\documentclass[10pt,letterpaper]{article}
\usepackage{opex3}
\usepackage{color}
\usepackage{amsthm}
\usepackage{mathrsfs}
\usepackage{times}
\usepackage{epsfig}
\usepackage{graphicx}
\usepackage{amsmath}
\usepackage{amssymb}
\usepackage{subfigure}
\usepackage{bm}
\usepackage{multirow}
\usepackage{array}
\usepackage{url}
\usepackage{slashbox}
\usepackage{appendix}
\usepackage{cite}
\usepackage{booktabs}
\usepackage{algpseudocode}
\usepackage{algorithm}
\usepackage{arydshln}
\usepackage{fancyhdr}
\DeclareMathOperator*{\argmin}{argmin}

\usepackage{fancyhdr}
\renewcommand{\thispagestyle}[1]{}  
\begin{document}
 \pagestyle{fancy} \lfoot{\footnotesize {This paper has been published in Optics Express, Vol. 23, Issue 3, pp. 2220-2239. The final version is available on 
 http://dx.doi.org/10.1364/OE.23.002220. Please refer to that version when citing this paper.}}

\title{Pixel-wise orthogonal decomposition for color illumination invariant and shadow-free image}

\author{Liangqiong Qu,$^{1,2}$ Jiandong Tian,$^{1}$ Zhi Han,$^{1}$ and Yandong Tang$^{1,*}$}

\address{$^1$ State Key Laboratory of Robotics, Shenyang Institute of Automation, Chinese Academy of Sciences, Shenyang, 110016, China \\
$^2$University of Chinese Academy of Sciences, Beijing, 100049, China}

\email{$^*$Corresponding author: ytang@sia.cn} %% email address is required



\begin{abstract}
In this paper, we propose a novel, effective and fast method to obtain a color illumination invariant and shadow-free image
  from a single outdoor image. Different from state-of-the-art methods for shadow-free image that either need shadow detection or statistical
  learning, we set up a linear equation set for each pixel value vector based on physically-based shadow invariants,
  deduce a pixel-wise orthogonal decomposition for its solutions, and then get an illumination invariant vector for each
  pixel value vector on an image. The illumination invariant vector is the unique particular solution of the linear equation set,
  which is orthogonal to its free solutions. With this illumination invariant vector and Lab color space, we propose an algorithm
  to generate a shadow-free image which well preserves the texture and color information of the original image. A series of experiments on
   a diverse set of outdoor images and the comparisons with the state-of-the-art methods validate our method.
\end{abstract}

 % REPLACE WITH CORRECT OCIS CODES FOR YOUR ARTICLE, MINIMUM OF TWO; Avoid using the OCIS codes for ¡°General¡± or ¡°General science¡± whenever possible.
%For a complete list of OCIS codes, visit: http://www.opticsinfobase.org/submit/ocis/

%%%%%%%%%%%%%%%%%%%%%%% References %%%%%%%%%%%%%%%%%%%%%%%%%

%%%%%%%%%%%%%%%%%%%%%%%%%%  body  %%%%%%%%%%%%%%%%%%%%%%%%%%
\section{Introduction}
Shadows, a physical phenomenon observed in most natural scenes, are problematic to many computer vision algorithms.
They often bring some wrong results in edge detection, segmentation, target recognition and tracking. For these reasons,
shadow processing and illumination invariant on an image are of great practical significance and have attracted a
lot of attentions \cite{land1971,weiss2001,finlayson2001color,lin2006separation,jung2011detecting,barron2013}. Some previous works dealt with this problem by
utilizing additional information from multiple images \cite{finlayson2007,laffont2012,laffont2013}, time-lapse  image sequences \cite{weiss2001,matsushita2004,huerta2009}
or user inputs \cite{wu2005,bousseau2009,laffont2013}. Although these methods achieve impressive results for illumination invariant and
shadow-free image, they may not be applied automatically to a single image. In general, a robust method for illumination invariant
and shadow-free on a single image is more favored and challenging. The recent works related to a shadow-free image on a
single image can be divided into two categories: methods that need shadow detection and methods that do not need shadow detection.

\textbf{Methods with shadow detection.} Most works for shadow-free image often undergo two basic steps: shadow detection and shadow removal \cite{arbel2011}.
A fair proportion of shadow detection methods are based on shadow features with statistical learning \cite{guo2012,lalonde2010,tappen2005,zhu2010}.
The most commonly used shadow features are intensity, histograms, texture, color ratio, geometry property, and gradient.
Lalonde et al. \cite{lalonde2010} obtained shadow removal results in consumer-grade photographs by a conditional random field (CRF) learning approach
on color ratio, texture and intensity distributions. Based on initial segmentation, the method proposed by Guo et al. \cite{guo2012} detects shadows
using SVM on shadow features (e.g, texture, chromaticity, and pairwise relation) and then relights each pixel to obtain a shadow-free image.
These shadow features may not be robust enough in some applications. For example, shadowed regions are often dark, with less texture and little gradient,
but some non-shadowed regions may also have similar characteristics. Due to the lack of robust shadow features, these shadow feature and learning
based methods are not always effective, especially in soft shadow regions or some complex scenarios. Besides, all these learning-based methods are
usually time-consuming since large complex features are needed to be analyzed for classifiers. Some other works for shadow-free image with
shadow detection are based on computing physically-based illumination invariants \cite{finlayson2006,finlayson2009,tian2009,Tian2011}.
With a physically-based illumination invariant, Finlayson et al. \cite{finlayson2006,finlayson2009} proposed a method to obtain a shadow-free
image by shadow edge detection and solving a Poisson\textquoteright s equation. In \cite{Tian2011}, we deduced a linear relationship between pixel values
of a surface in and out of shadow regions, and then proposed a shadow removal method based on shadow edge detection.
In the condition of correct shadow edge detection, the method proposed by Arbel and Hel-Or \cite{arbel2011} yields a shadow-free image by
finding scale factors to cancel the effect of shadows. All these mentioned methods require shadow edges to be correctly detected. However,
automatic edge detection is known to be non-robust \cite{jain1995}, especially in complex scenarios and texture (such as the shadows of leaves and grasses).
The defect of these shadow detection based methods is revealed in \cite{yang2012}. Even when the shadows are detected rightly, retaining the
original color and texture in a shadow-free image is also a challenging problem \cite{arbel2011,liu2008texture}.

\textbf{Methods without shadow detection.} Another effective approach for shadow-free image is
intrinsic image via image decomposition. Tappen et al. \cite{tappen2005} proposed a learning-based approach to predict the derivatives of reflectance
and illumination images. Although this method successfully separates reflectance and shading for a given illumination direction used for training, it is not
designed for arbitrary lighting directions \cite{yang2012}. Assuming that the neighboring pixels that have similar intensity values should have similar reflectance,
Shen et al. \cite{shen2011} proposed both an automatic and user scribbles method to separate intrinsic image from a single image with optimization.
Although their automatic method is able to extract most of the illumination variations, residual shadows remain in the reflectance image and they stated that user scribbles cannot improve the result significantly.
Finlayson et al. \cite{finlayson2009} derived a grayscale shadow-free image by finding the special direction in a 2-D chromaticity feature space.
In \cite{Tian2011}, we obtained a grayscale shadow-free image based on a physical deduced linear model from the view of atmospheric transmittance effects.
These derived grayscale shadow-free images have some limitations in real applications due
to the loss of color information. Applying narrowband camera sensors in deduction, Maxwell et al. \cite{maxwell2008} found that a 2-D illumination invariant image can be obtained
by projecting the colors (in log-RGB) onto a plane. In their method, the projecting normal vector is determined either from a user input or a 2-D
entropy searching, which may have some limitations in real applications. With the assumption of chromatic surfaces, Planckian lighting and narrowband camera sensors, Yang et al. \cite{yang2012} obtained a shadow-free image
using bilateral filtering without shadow detection. The comparison results in \cite{yang2012} shows that this method for shadow-free image outperforms Finlayson\textquoteright s\cite{finlayson2006} (a classical method for shadow-free image that need shadow detection).
However, based solely on chromaticity, this method may not be valid on neutral regions when the neighboring regions of the image are also neutral \cite{yang2012}.
The experimental results presented in section $\mathbf{4}$ of this paper also
show that bilateral filtering often makes the method in \cite{yang2012} fail to recover a shadow-free image from the
image with darker shadows or in which shadows pervade in a scenario.
\begin{figure}[h]
\centering
\includegraphics[width=0.98\linewidth]{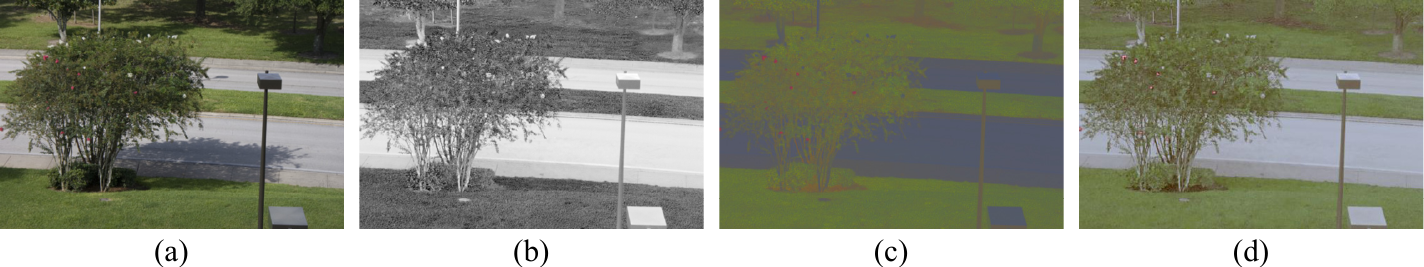}
   \caption{One experimental result of our algorithm. (a) Original image. (b) One of the three grayscale illumination invariant images. (c) Color illumination
   invariant image with our  pixel-wise orthogonal decomposition. (d) Shadow-free image after color restoration.}
\label{fig:illustration}
\end{figure}

In this paper, we propose a novel, effective and efficient method to obtain a shadow-free image from a single outdoor image without feature
extraction or shadow detection. The research is based on our previous work of grayscale illumination
invariant image derivation from a physically-based shadow linear model \cite{Tian2011}.
With three different grayscale illumination invariant images (e.g., Fig. \ref{fig:illustration}(b)) from this linear model \cite{Tian2011},
we apply them to set up three linear equations for each pixel value vector. The solution space of these linear equations is decomposed into
one-dimensional nullspace (free solution space) and particular solution space. The free solution, representation of illuminants ratio, is only determined by illumination condition.
The particular solution, which is perpendicular to free solutions, is unique and illumination invariant.
It retains the basic color information of the original pixel value vector. Therefore, we name this process as pixel-wise orthogonal decomposition for color
illumination invariant. Due to that this color illumination invariant image (e.g., Fig. \ref{fig:illustration}(c)) still has some color distortion compared
with the original image, combining color illumination invariant image and Lab color space, we propose an algorithm to generate a shadow-free image (e.g.,
Fig. \ref{fig:illustration}(d)), which better preserves the color and texture information of the original image.

The main contributions of this paper are: (1) We propose a new pixel-wise orthogonal decomposition method to obtain a shadow-free image without feature extraction or shadow
detection. (2) Our method requires only once pixel-wise calculation on an image and can run in real time. (3) This pixel-wise orthogonal decomposition retains the original
texture well. (4) The color illumination invariant image yields an identical reflectivity result regardless of illumination condition, which is both invariant for different illumination
conditions and shadows.

The rest of the paper is organized as follows. In Section $\mathbf{2}$, we give a brief introduction of the derivation of grayscale illumination
invariant image and then present our orthogonal decomposition method for color illumination invariant image from a single image. In section
$\mathbf{3}$, we present a color restoration method for the generation of shadow-free image from color illumination invariant image.
Some quantitative and qualitative experiments on a diverse set of shadow images are given in
Section $\mathbf{4}$. We end this paper with a brief discussion and conclusion in section $\mathbf{5}$.
\section{Pixel-wise orthogonal decomposition and color illumination invariant image}
In this section, we first give a brief introduction of our previous work of shadow linear model
and the derivation of relevant grayscale illumination invariant image. Then we present our pixel-wise orthogonal
decomposition method for color illumination invariant image from a single image.

%-------------------------------------------------------------------------
\subsection{Pixel-wise orthogonal decomposition}
\subsubsection{Shadow linear model on an image \cite{Tian2011}} For visible light whose spectrum ranges
from 400nm to 700nm, given spectral power distributions (SPD) of illumination $E(\lambda)$, object reflectance $S(\lambda)$
and sRGB color matching functions $Q(\lambda)$, the RGB tristimulus values ($\mathscr{F}_H$) in sRGB color space are computed by,
\begin{equation}
{\mathscr{F}_H} = \int_{400}^{700} {E(\lambda )S(\lambda ){Q_H}(\lambda )} d\lambda
\end{equation}
where $H=\{R, G, B\}$ represent red, green, and blue channel respectively. Light emitted from the sun will be scattered by
atmospheric transmittance effects which causes the incident light to be split into direct sunlight and diffuse skylight.
The illumination for shadow is skylight while for non-shadow background is both sunlight and skylight. In clear weather
when shadows most probably take place, the SPD of daylight and skylight show strong regularities, and they are mainly
controlled by sun angles. Based on these principals, it is revealed that the sRGB tristimulus values of a surface illuminated
by daylight are proportional to those of the same surface illuminated by
skylight in each of the three color channels, i.e.,

\begin{equation}
\log ({F_H} + 14) = \frac{{\log ({K_H})}}{{2.4}} + \log ({f_H} + 14)
 \label{Linear}
\end{equation}
where $F_H$ denotes the RGB values of a surface in non-shadow area and $f_H$ denotes the RGB values for the same surface in shadow area. The proportional
coefficients $K_H$ are independent of reflectance and are approximately equal to constants determined by Eq. (\ref{K_H}).
\begin{equation}\label{K_H}
{K_H} = \mathop {\argmin}\limits_{{K_H}} \sum\limits_{\lambda  = 400}^{700} {\left| {{Q_H}(\lambda ) \cdot ({E_{day}}(\lambda ) - {K_H} \cdot {E_{sky}}(\lambda ))} \right|}
\end{equation}
where $E_{day}$ and $E_{sky}$ represent the SPD of daylight and skylight respectively. In our experiment,
this SPD of daylight and skylight were obtained by calculating the mean values of the SPD of daylight and skylight measured by
an Avantes USB 2.0 spectrometer under different sun angles on ten different days in sunny weather. This optimization problem of ${K_H}$ is solved by setting the sampling stepsize of in Eq. (\ref{K_H}) as 5nm and the stepsize of ${K_H}$ as 0.01.

From Eq. (\ref{Linear}), the following equation holds.
\begin{equation}
 \begin{split}
     & \log ({F_R} + 14){\rm{ + }}\log ({F_G} + 14) - {\beta _1} \cdot \log ({F_B} + 14)\\
     & =\log ({f_R} + 14) + \log ({f_G} + 14) - {\beta _1} \cdot \log ({f_B} + 14)
 \end{split}
 \label{gray-invariant1}
\end{equation}
where
\begin{equation}\label{belta1}
  {\beta _1} = \frac{{\log ({K_R}) + \log ({K_G})}}{{\log ({K_B})}}
\end{equation}
For a pixel in an image, Eq. (\ref{gray-invariant1}) represents a shadow invariant. For an arbitrary pixel and its RGB value vector,
${({v_R},{v_G},{v_B})^T}$, $\log ({v_R} + 14){\rm{ + }}\log ({v_G} + 14) - {\beta _1} \cdot \log ({v_B} + 14)$ keeps the same no matter whether
the pixel is on shadow region or not.
Apparently, a grayscale illumination invariant image is obtained.
Two results of this type of illumination invariant image are shown in Fig. \ref{fig:gray}.

\begin{figure}[h]
\centering
\includegraphics[width=0.98\linewidth]{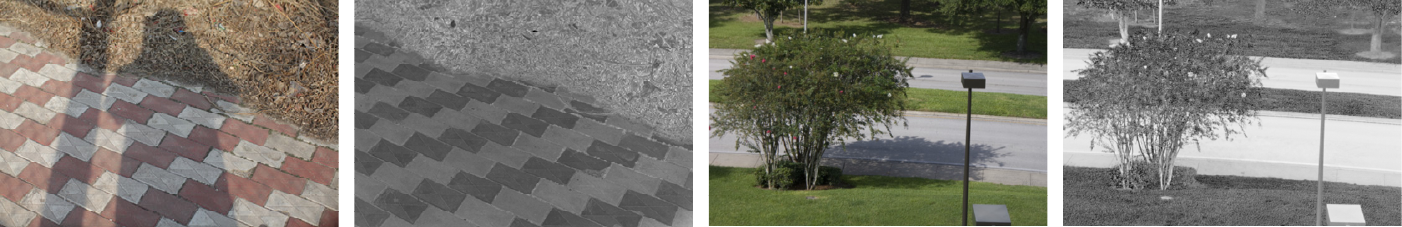}
   \caption{Grayscale illumination invariant image results based on the linear model}
\label{fig:gray}
\end{figure}

\subsubsection{Pixel-wise orthogonal decomposition}
The grayscale illumination invariant image has some limitations in some
applications due to the loss of color and contrast information. We expect to get an illumination (including shadow) invariant image, which can retain
the color information of the original image and appear the same as the original image without shadow. For this purpose, based on the model in
previous section, we deduce a pixel-wise orthogonal decomposition and obtain an illumination invariant image which retains the basic color information
of the original image.

For a RGB value vector ${({v_R},{v_G},{v_B})^T}$, we first define the following grayscale illumination invariant values $I_1$ according to Eq. (\ref{gray-invariant1}):
\begin{equation}\label{Invariant1}
  \begin{array}{rl}
I_1 &= \log ({F_R} + 14){\rm{ + }} \log ({F_G} + 14){\rm{ - }}{\beta _1} \cdot  \log ({F_B} + 14) \\
  &= \log ({f_R} + 14){\rm{ + }} \log ({f_G} + 14){\rm{ - }}{\beta _1} \cdot \log ({f_B} + 14)  \\
  &=  \log ({v_R}+14){\rm{ + }}\log ({v_G} +14)- {\beta _1} \cdot \log ({v_B}+14)
 \end{array}
\end{equation}

Similar to Eq. (\ref{Invariant1}), we can obtain two other grayscale illumination invariant values $I_2$ and $I_3$,
\begin{equation}\label{gray-invariant2}
  \begin{array}{rl}
I_2 &= \log ({F_R} + 14){\rm{ - }}{\beta _2} \cdot \log ({F_G} + 14){\rm{ + }}\log ({F_B} + 14) \\
  &= \log ({f_R} + 14){\rm{ - }}{\beta _2} \cdot \log ({f_G} + 14){\rm{ + }}\log ({f_B} + 14)  \\
  &=  \log ({v_R}+14){\rm{ - }}{\beta _2} \cdot \log ({v_G} +14){\rm{ + }} \log ({v_B}+14)
 \end{array}
\end{equation}
\begin{equation}\label{gray-invariant3}
  \begin{array}{rl}
I_3 &= {\rm{ - }}{\beta _3} \cdot \log ({F_R} + 14) + \log ({F_G} + 14){\rm{ + }}\log ({F_B} + 14) \\
  &= {\rm{ - }}{\beta _3} \cdot \log ({f_R} + 14) + \log ({f_G} + 14){\rm{ + }}\log ({f_B} + 14) \\
  &=  {\rm{ - }}{\beta _3} \cdot \log ({v_R}+14) + \log ({v_G} +14){\rm{ + }}\log ({v_B}+14)
 \end{array}
\end{equation}
where
\begin{equation}\label{belta23}
 {\beta _2} = \frac{{\log ({K_R}) + \log ({K_B})}}{{\log ({K_G})}}, {\beta _3} = \frac{{\log ({K_G}) + \log ({K_B})}}{{\log ({K_R})}}
\end{equation}

For convenience, let $\bm{u} = {({u_R},{u_G},{u_B})^T}$ defines a Log-RGB value vector of a pixel, where ${u_H} = log({v_H} + 14), H = \{R,G,B\}$.
From Eqs. (\ref{gray-invariant1}), (\ref{gray-invariant2}), and (\ref{gray-invariant3}), we have the following linear equations,
\begin{equation}\label{solution1}
 \left\{ {\begin{array}{*{20}{c}}
   {{u_R}{\rm{ + }}{u_G} - {\beta _1} \cdot {u_B} = {I_1}}  \\
   {{u_R} - {\beta _2} \cdot {u_G}{\rm{ + }}{u_B} = {I_2}}  \\
   { - {\beta _3} \cdot {u_R}{\rm{ + }}{u_G}{\rm{ + }}{u_B} = {I_3}}  \\
\end{array}} \right.
\end{equation}
where ${I_i},i = 1,2,3,$ correspond to the invariant values calculated by Eqs. (\ref{Invariant1}), (\ref{gray-invariant2}), and (\ref{gray-invariant3})
respectively. The matrix format of Eq. (\ref{solution1}) is as following:
\begin{equation}\label{maineq1}
A\bm{u} = \bm{I}
\end{equation}
where
 $A = \left[ {\begin{array}{*{20}{c}}
1&1&{ - {\beta _1}}\\
1&{ - {\beta _2}}&1\\
{ - {\beta _3}}&1&1
\end{array}} \right]$, $\bm{u} = {({u_R},{u_G},{u_B})^T}$, and $\bm{I} = {({I_1},{I_2},{I_3})^T}$. According to the definitions and calculations of
${\beta _1},{\beta _2}$, and ${\beta _3}$, we have
\begin{equation}
  2 + {\beta _1} + {\beta _2}{\rm{ + }}{\beta _3} - {\beta _1}{\beta _2}{\beta _3} = 0
\end{equation}
which leads to $rank(A) = 2$. For a given image, from algebraic theory, we know that Eq. (\ref{maineq1}) has infinite number of solutions.
The solution space can be decomposed into a particular solution plus one-dimensional nullspace solutions.
For a pixel on an image, we only know one of the solutions, i.e., pixel Log-RGB value vector. We aim to find one solution of Eq. (\ref{maineq1}) which is
illumination invariant. In the following, we prove that this solution exists and is unique.

From algebraic theory, any solution $\bm{u}$ of Eq. (\ref{maineq1}), can be expressed as following:
\begin{equation}\label{decomposition}
  \bm{u} = \bm{u_s} + \alpha \bm{u_0}
\end{equation}
where $\bm{u_s}$ is an arbitrary particular solution and $\bm{u_0}$ is the normalized free solution of Eq. (\ref{maineq1}), such that
$A\bm{u_0} = 0$ and ${\rm{ }}{\left\| \bm{u_0} \right\|} = 1$, $\alpha  \in R$. The symbol $\left\|  \cdot  \right\|$ denotes ${L^2}$ norm.
One of the free solution is solved in Formula \ref{free_solution},
\begin{equation}\label{free_solution}
\begin{array}{cl}
 \bm{u_0}^\prime& = \left( {{\beta _1}{\beta _2}{\rm{ - }}1,1 + {\beta _1},1 + {\beta _2}} \right) \\
  & = \log ({K_R}{K_G}{K_B}) \cdot (\log ({K_R}),\log ({K_G}),\log ({K_{\rm{B}}})) \\
 \end{array}
\end{equation}
From Formula \ref{free_solution}, the normalized free solution $\bm{u_0}$ can be calculated as following:
\begin{equation}\label{norm_free_solution}
  \bm{u_0} = \frac{1}{{\left\| {\bm{u_0}^\prime} \right\|}}\bm{u_0}^\prime
\end{equation}

Here the free solution has no relationship with the image itself but is determined by matrix $A$, i.e. illumination condition.
Formula \ref{free_solution} further reveals that this free solution is a representation of
the ratio of illuminants in imaging environments.

Given a particular solution $\bm{u_s}$ and the normalized free solution $\bm{u_0}$ of Eq. (\ref{maineq1}), according to Eq. (\ref{decomposition}) and algebraic theory, we define $\bm{u_p}$ as
\begin{equation}
  \left\{ {\begin{array}{*{20}{c}}
   {\bm{u_p} = \bm{u_s} + {\alpha _p}\bm{u_0}}  \\
   {{\alpha _p} =   - \langle \bm{u_s},\bm{u_0}\rangle }  \\
\end{array}} \right.
\end{equation}
It is a particular solution of Eq. (\ref{maineq1}), where $\langle  \cdot , \cdot \rangle $ denotes vector inner product. It can be proved that $\bm{u_p}$,
satisfying $\bm{u_p} \bot \bm{u_0}$, is the unique particular solution of Eq. (\ref{maineq1}) (The proof of the uniqueness of $\bm{u_p}$ is provided
in \textbf{Appendix} A).
For a Log-RGB value vector $\bm{u}$, whether it is on shadow or not, it is the solution of Eq. (\ref{maineq1}) and can be decomposed as
\begin{equation}\label{decomposition1}
  \bm{u} = \bm{u_p} + \alpha \bm{u_0}
\end{equation}
where $\alpha= \langle \bm{u},\bm{u_0}\rangle$ and $\bm{u_p}$, $\bm{u_p} = \bm{u} -  \alpha \bm{u_0}$, is the unique particular solution of Eq. (\ref{maineq1}).
Due to  $\bm{u_p} \bot \bm{u_0}$, we call this process as pixel-wise orthogonal decomposition. For a given pixel, no matter how different the values of the pixel
are with different illumination conditions, this pixel-wise orthogonal decomposition will yield a unique particular solution $\bm{u_p}$.

From the deducing procedure above and the orthogonal decomposition, we know that $\bm{u_0}$, representation of illuminants ratio, is only related with
${({\beta _1},{\beta _2},{\beta _3})^T}$ and $\bm{u_p}$, perpendicular to $\bm{u_0}$, is illumination invariant.
It means that for a pixel with Log-RGB value vector $\bm{u}$, no matter whether it is on shadow or not, $\bm{u_p}$ is invariant and only $\alpha $
reflects the variation of pixel RGB values caused by shadow or different illuminants.
The different solutions caused by different $\alpha$ component
in Eq. (\ref{decomposition1}) can be regarded as Log-RGB values of a pixel on different illumination conditions.
Therefore, given an image under multiple illumination conditions, for each pixel with Log-RGB value vector $\bm{u}$ of this image,
we can directly calculate its illumination invariant vector $\bm{u_p}$ as
\begin{equation}\label{decompose up}
  \left\{ {\begin{array}{*{20}{c}}
   {\bm{u_p} = \bm{u} + {\alpha _p}\bm{u_0}}  \\
   {{\alpha _p} =   - \langle \bm{u},\bm{u_0}\rangle }  \\
\end{array}} \right.
\end{equation}
By exponential transform on $\bm{u_p}$ for each pixel, we can get an illumination invariant image which retains basic color information of the original image in RGB color space.

\begin{figure}[h]
\centering
\includegraphics[width=0.58\linewidth]{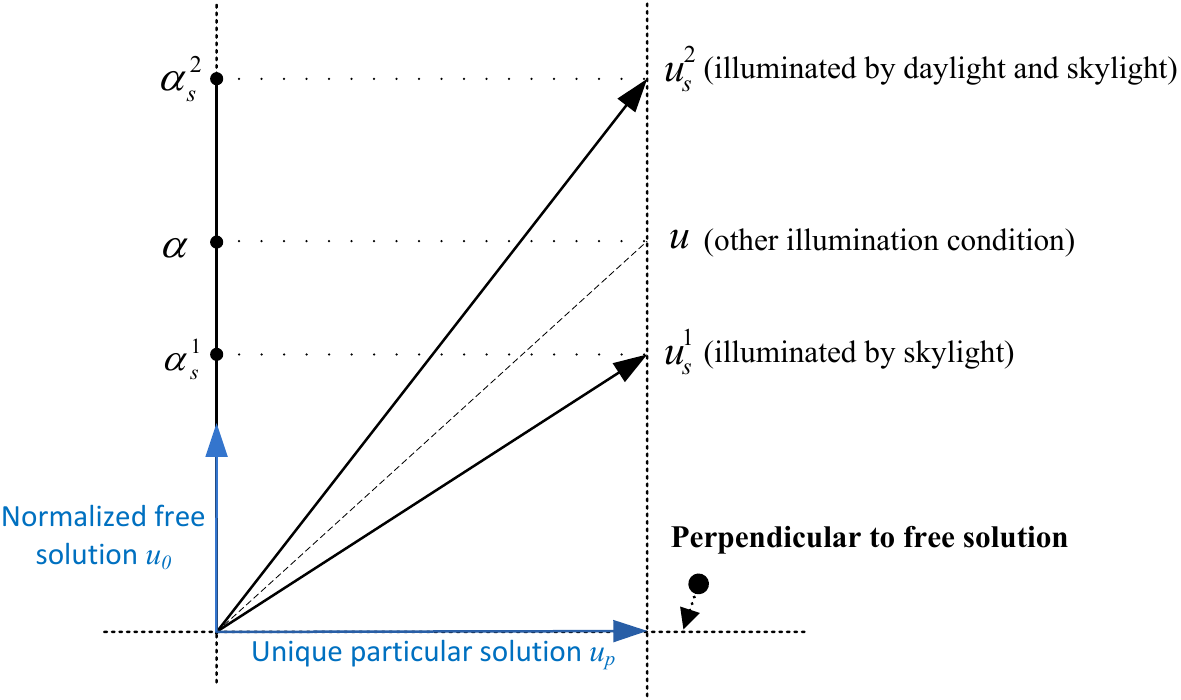}
   \caption{An illustration of pixel-wise orthogonal decomposition.}
\label{fig:geometry}
\end{figure}
This pixel-wise orthogonal decomposition process can be shown more clearly from the illustration in Fig. \ref{fig:geometry}.
For a pixel, $\bm{u_s}^1$ and $\bm{u_s}^2$ denote its log-RGB vector within shadow and log-RGB vector without shadow respectively, which are the solutions of Eq. (\ref{maineq1}). Here we denote any solution of Eq. (\ref{maineq1}) by vector $\bm{u}$.
According to our orthogonal decomposition Eqs. (\ref{free_solution}), (\ref{norm_free_solution}), and (\ref{decomposition1}), $\bm{u_s}^1$,
$\bm{u_s}^2$, and $\bm{u}$ can be projected along the vector $\bm{u_0}$
into an specific solution of Eq. (\ref{decomposition}), $\bm{u_p}$, which is perpendicular to $\bm{u_0}$.
Therefore, for a pixel on an image, it means that, no matter how different the values of the pixel are with different illumination conditions (within shadow,
without shadow or other illuminated), we can get an unique invariant $\bm{u_p}$ by our orthogonal decomposition.

%---------------------------------------------------------------------------------
\subsection{Analysis of Parameters ${\beta _1},{\beta _2},{\beta _3}$ and color illumination invariant image}
The estimation of parameters ${\beta _1},{\beta _2},{\beta _3}$ is important in determining whether our pixel-wise orthogonal decomposition
works well or not. Luckily, from Eqs. (\ref{belta1}), (\ref{belta23}), and (\ref{K_H}), we know that these parameters are reflectance irrelevant
and only determined by the illumination condition (i.e, the SPD of daylight and skylight), such as zenith angle (time) and aerosol (weather). In clear weather that shadows mostly take place,
these parameters are mainly determined by zenith angle.
The SPD can be measured by spectrometer or calculated in
atmosphere optics (software like MODTRAN and SMARTS2) with zenith angle and aerosol (for weathers) in advance.
We have calculated parameters ${\beta_1},{\beta_2},{\beta_3}$ in sunny weather at different sun angles from 9:00 AM till 4:00 PM
in our city and some of the values are shown in Tab.\ref{table:belta}. The mean values in Tab.\ref{table:belta} are calculated by the mean SPD of daylight and skylight at
$20^\circ$ to $70^\circ$.
From Tab.\ref{table:belta}, we notice that the values of parameters ${\beta_1},{\beta_2},{\beta_3}$ from sun angle $20^\circ$ to sun angle $80^\circ$ are quite stable.
In real applications, the sun angle (time) is known and the parameters can then be determined. In the condition that we
have no information about at what sun angles the pictures were taken, the stability of parameters shows that the mean values
for ${\beta_1},{\beta_2},{\beta_3}$ at $20^\circ$ to $70^\circ$ sun angles can be used for most situations in our decomposition,
except some images taken at twilight, i.e., at sunrise or sunset, when sun angle is less than $10^\circ$. Also, for more precise parameter values, an automatic
entropy minimization method similar to Finlayson et al. \cite{finlayson2009} is employed.
\begin{table*}[!t]
\centering
\caption{Parameters (PRMs) from representative sun angles}
\label{table:belta}
\begin{tabular}{c||c|c|c|c|c|c|c|c}
\hline
\backslashbox{PRMs}{Angle}  & $20^\circ$ & $30^\circ$ &  $40^\circ$ & $50^\circ$ & $60^\circ$ & $70^\circ$ & $80^\circ$ & Mean \\
\hline\hline
$\beta_1$  & 2.353 &  2.321&2.299 & 2.371 &  2.648 & 2.520 & 2.473 &2.557  \\
$\beta_2$ &1.963& 1.963 &1.977&1.982&1.925&1.996 & 1.985 &1.889\\
$\beta_3$ &1.745& 1.767&1.770&1.716&1.604&1.617& 1.652 &1.682\\
\hline
\end{tabular}
\end{table*}
In addition to the factors (zenith angle and aerosol) mentioned above that may affect parameters ${\beta_1},{\beta_2},{\beta_3}$,
some other aspects are also worthy of attention. For example, cloud and fog will affect the SPD of daylight and skylight and further affect the parameters
${\beta_1},{\beta_2},{\beta_3}$ in our algorithm.
From clear weather to cloudy, as the clouds increase (aerosol increases), shadows vary from strong to weak.
When shadows are weak, the error of the color illumination invariant results caused by parameters (${\beta_1},{\beta_2},{\beta_3}$) deviation from cloud is also weaker.
Therefore, the influence caused by the cloud distribution would be small and can be corrected by automatic entropy minimization \cite{finlayson2009}.
This similar situation also goes for foggy weather.
\begin{figure}[h]
\centering
\includegraphics[width=0.7\linewidth]{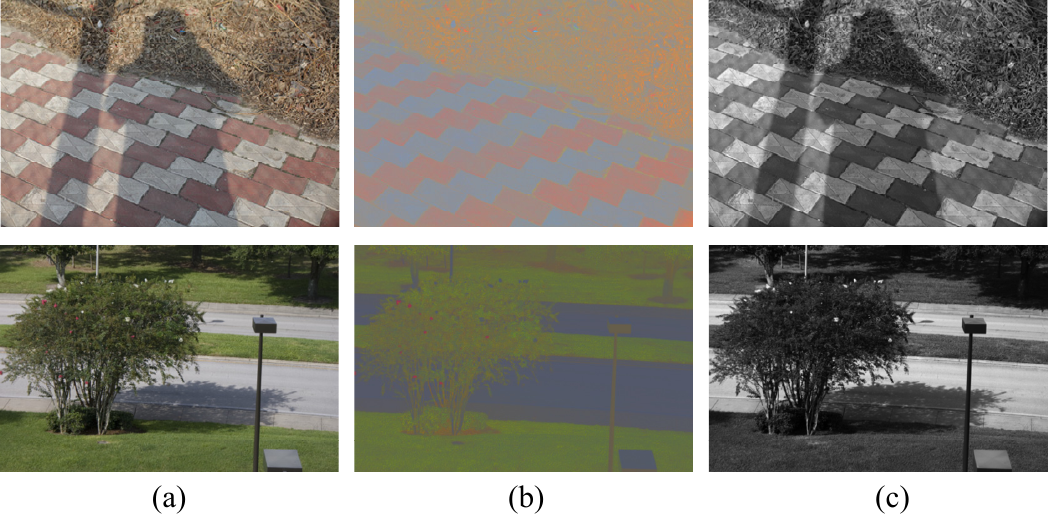}
   \caption{Pixel-wise orthogonal decomposition on images. Left: Original image. Middle: color illumination invariant image. Right: $\alpha$ information.}
\label{fig:par}
\end{figure}

The above analysis shows that our pixel-wise orthogonal decomposition can work well on most situations. Our experimental results on many images in different
scenes and reflectance also support this analysis. Two original images, their color illumination invariant images and $\alpha$ information are presented in
Fig. \ref{fig:par}. The shadows disappear and the main color information is maintained in our color illumination invariant images.

Compared with prior works \cite{drew2003,finlayson2009,maxwell2008}, our 2-D color illumination invariant image has three main advantages:

(1) Our
2-D color illumination invariant image is a true 2-D image
which can be automatically obtained through a simple pixel-wise orthogonal decomposition. In \cite{finlayson2009,drew2003}, Finlayson and Drew obtained
a 2-D chromaticity image by projecting the 1-D chromaticity back to 2-D via keeping the vector components. In \cite{maxwell2008},  Maxwell et al. stated that this 2-D chromaticity image is inherently one
dimensional. Applying narrowband camera sensors in deduction, Maxwell et al. \cite{maxwell2008} successfully obtained a truly 2-D color illumination invariant image by projecting the colors (in log-RGB) onto a plane based on Shafer\textquoteright s dichromatic reflection model \cite{shafer1985}. Their 2-D color illumination invariant image is obtained either from a user input for projecting normal vector or an 2-D entropy searching. This user input or 2-D entropy search of the angle of the normal vector in a three dimensional space may bring some difficulties in real applications, especially for some images with small fragment shadows and complex reflectance. While our 2-D color illumination invariant image can be directly obtained by measuring the SPD of daylight and skylight or
simply giving the information about at what sun angles the pictures were taken (the time when the pictures were taken).
Even when these information is unknown to the users, a mean value of SPD information is suitable for most situations. For the correction of the errors caused by the mean values, a simple 1-D entropy searching is used to get optimized parameter values among all the given parameters that be measured in different sun angles in advance.

(2) Based on a rigorous mathematical deducing, our 2-D color illumination invariant image has an explicit and simple mathematical expression. This expression and fast pixel-wise calculation of our 2-D color illumination invariant image can be applied directly to the real-time applications.

(3) From the view of atmospheric transmittance
effects, we do not apply the assumption of Planckian lighting and narrowband camera sensors in our derivation, which may be more practical in real situation.
%------------------------------------------------------------------------
\section{Color restoration and shadow-free image}
By the orthogonal decomposition on an image, some image color information gets lost (Fig. \ref{fig:par}). As shown in Tab.\ref{table:belta} and Eqs. (\ref{free_solution}) and (\ref{norm_free_solution}), the directions of free solution in different sun
angles (from sun angle $20^\circ$ to sun angle $80^\circ$) are all near to $(1,1,1)^T$. From the orthogonal decomposition equation, we have, if $\bm{u} \approx \alpha \bm{u_0}$,
the illumination invariant $\bm{u_p} \approx {(0,0,0)^T}$. It implies that the closer to the neutral surface on RGB space a pixel RGB value vector is, the more color information
is got lost. To restore this loss, we present a color restoration method to generate a shadow-free image by correcting color vectors on the illumination invariant image which are
near RGB color neutral surface.

Let $\bm{v}(x,y) = {({v_R}(x,y),{v_G}(x,y),{v_B}(x,y))^T}$ be an RGB image, where $x = 1,2,...,M; y = 1,2,...,N$; $M$ and $N$ denote the image width and height respectively. We define
\begin{equation}
\begin{array}{l}
\bm{u}(x,y) = (\log ({v_R}(x,y) + 14),\log ({v_G}(x,y) + 14),\\
\begin{array}{*{20}{c}}
{\begin{array}{*{20}{c}}
{\begin{array}{*{20}{c}}
{\begin{array}{*{20}{c}}
{\begin{array}{*{20}{c}}
{}&{}
\end{array}}&{}
\end{array}}&{}
\end{array}}&{}
\end{array}}&{}
\end{array}\begin{array}{*{20}{c}}
{}&{}
\end{array}\log ({v_B}(x,y) + 14){)^T}
\end{array}
\end{equation}
\begin{algorithm}[b]
  \caption{Algorithm pipeline: Pixel-wise orthogonal decomposition for shadow-free image}
  \label{algorithm1}
  \begin{algorithmic}[0]
\State  \hspace{-12pt} \textbf{Input:} Original image $I$, Coefficient $K_H$;
    \State \hspace{-12pt} \textbf{Output:} Shadow-free image  $O$;
    \State \hspace{-12pt} \textbf{Some notations:} \\
$\bm{u}$, $\bm{u_c}$, $\bm{u_p}$: Image in the Log-RGB color space;   \\
$\bm{u_p}^{RGB}$, $\bm{u_c}^{RGB}$: Image in the RGB color space, the superscript $RGB$ indicates image in the RGB color space;   \\
$\bm{u_p}^{Lab}$, $\bm{u_c}^{Lab}$, $\bm{u_f}^{Lab}$: Image in the Lab color space, the subscript $Lab$ indicates images in the Lab color space;
     \State \hspace{-12pt} \textbf{Algorithm:}
  \end{algorithmic}
  \begin{algorithmic}[1]
  \vspace{-1mm}
    \State Perform a logarithmic transformation on $I$ to get the Log-RGB image $\bm{u}$;
    \State Make a pixel-wise orthogonal decomposition on $\bm{u}$ according to Eq. (\ref{decompose up}) to get the color illumination invariant values $\bm{u_p}$, where we have
    $\bm{u_p} = \bm{u} -  < \bm{u},\bm{u_0} > \bm{u_0}$;
    \State Calculate the parameter vector $T$ according to Eqs. (\ref{S}) and (\ref{T}). This parameter vector is used to correct the colors in the original image that lies close to the neutral surface;
    \State Perform color correction on $\bm{u_p}$ according to Eq. (\ref{color-restoration}) to get $\bm{u_c}$;
    \State Transfer $\bm{u_p}$ and $\bm{u_c}$ back to the RGB color space by an exponentiation transformation
    and get $\bm{u_p}^{RGB}$ and $\bm{u_c}^{RGB}$ respectively. This image $\bm{u_p}^{RGB}$ has almost the same chrominance as original image
    and image $\bm{u_c}^{RGB}$ has almost the same luminance as original image;
    \State Convert $\bm{u_p}^{RGB}$ and $\bm{u_c}^{RGB}$ from RGB color space into Lab color space and get $\bm{u_p}^{Lab}$ and $\bm{u_c}^{Lab}$;
    \State Combine the luminance component of $\bm{u_c}^{Lab}$ and chrominance components of $\bm{u_p}^{Lab}$ to obtain the
    shadow-free image $\bm{u_f}^{Lab}$ in Lab space according to Eqs. (\ref{labcolor}) and (\ref{labunion});
    \State Convert $\bm{u_f}^{Lab}$ to the RGB color space to get the final shadow-free image $O$;
  \end{algorithmic}
\end{algorithm}
\begin{figure*}
\begin{center}
  \centering
  \includegraphics[width=0.9\textwidth,height=0.63\textwidth]{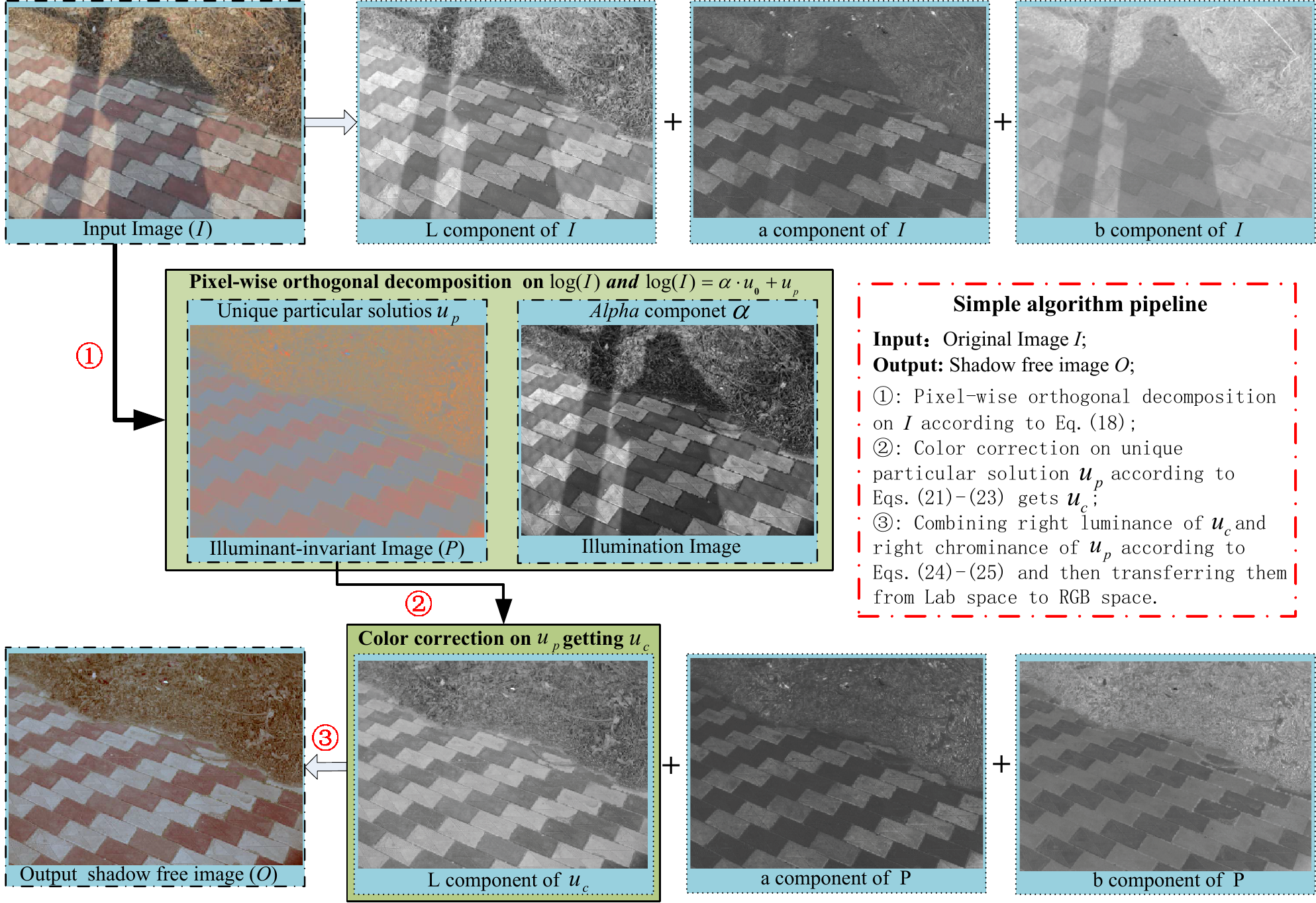}
    \caption{The stages of our algorithm implemented on a shadowed image. Note that the arrow in the first row represents the color space conversion,
    that is converting from RGB space to Lab space. These three components have no actual use in our algorithm. They are only used to act as the
    reference for the Lab component of our shadow-free image $O$.}
  \label{fig:flow_chart}
  \end{center}
\end{figure*}
With the orthogonal decomposition on $\bm{u}$ (to see Eq. (\ref{decomposition1})), we obtain
\begin{equation}
 \bm{u}(x,y) = \alpha (x,y)\bm{u_0} + \bm{u_p}(x,y)
\end{equation}
where $\bm{u_p}(x,y)$ is illumination invariant on pixel $(x,y)$, $\alpha (x,y) \in R$. For $\bm{u}$, we define the following pixel set, which
indicates pixels that are near RGB color neutral surface,
\begin{equation}\label{S}
\begin{split}
\begin{array}{l}
 S = \{ (x,y)|\left\| {\frac{{\bm{u}(x,y)}}{{\left\| {\bm{u}(x,y)} \right\|}} - \bm{u_0}} \right\| \le \varepsilon ; \\
 \begin{array}{*{20}{c}}
   {\begin{array}{*{20}{c}}
   {} & {}  \\
\end{array}} &\ {} & {}  \\
\end{array}x = 1,2,..,M;y = 1,2,...,N\}  \\
 \end{array}
\end{split}
\end{equation}
where parameter $\varepsilon$ is set to 0.15 empirically.
We then calculate the following parameter vector,
\begin{equation}\label{T}
  \bm{T} = \frac{1}{G}\sum\limits_{(x,y) \in S} {\left( {\bm{u_0} - \frac{{\bm{u}(x,y)}}{{\left\| {\bm{u}(x,y)} \right\|}}} \right)}
\end{equation}
where $G$ is the number of the pixels in set $S$. This parameter vector is used to measure the average deviation of the pixel vectors of image $\bm{u}$ in set $S$ from the vector $\bm{u_0}$. Then the color correction on $\bm{u_p}$ is calculated as following:
\begin{equation}\label{color-restoration}
\begin{array}{cl}
 \bm{u_c}(x,y) = & \parallel \bm{u_p}(x,y)\parallel  \cdot  \\
 & \left( {\frac{{\bm{u_p}(x,y)}}{{\parallel \bm{u_p}(x,y)\parallel }} + \frac{1}{{\kappa {{\left\| {\frac{{\bm{u}(x,y)}}{{\left\| {\bm{u}(x,y)} \right\|}} - \bm{u_0}} \right\|}^3} + 1}}T} \right) \\
 \end{array}
\end{equation}
From the formula above, if $\bm{u}(x,y) = \alpha (x,y)\bm{u_0}$, we have $\bm{u_c}(x,y) = \bm{u_p}(x,y) + T$.
This color correction aims to correct the color for pixels in set $S$. The role of function
${\frac{1}{{\kappa {{\left\| {\frac{{\bm{u}(x,y)}}{{\left\| {\bm{u}(x,y)} \right\|}} - \bm{u_0}} \right\|}^3} + 1}}}$ in Eq. (\ref{color-restoration})
is used to minimize the impact of this correction on pixels not in set $S$ and keep the correction smooth. The parameter $\kappa$ is set to 0.02 empirically.
With exponentiation, we transform $\bm{u_c}$ and $\bm{u_p}$ into RGB space and get $\bm{u_c}^{RGB}$ and $\bm{u_p}^{RGB}$ respectively, where image $\bm{u_p}^{RGB}$ has almost the
same chrominance as original image and image $\bm{u_c}^{RGB}$ has almost the same luminance as original image. This is verified by experimental results
on many images in different scenes and reflectance.

Then a shadow-free image is obtained by transferring the right luminance component of $\bm{u_c}^{RGB}$ to the color illumination invariant image.
We adopt Lab color space (L, the luminance component; a, b, two chrominance components) for this transfer. First we convert $\bm{u_c}^{RGB}$ and $\bm{u_p}^{RGB}$
from RGB space into Lab space and get $\bm{u_c}^{Lab}$ and $\bm{u_p}^{Lab}$ respectively,
\begin{equation}\label{labcolor}
\left\{ \begin{array}{l}
\bm{u_c}^{Lab}(x,y) = {({L_c}(x,y),{a_c}(x,y),{b_c}(x,y))^T}\\
\bm{u_p}^{Lab}(x,y) = {({L_p}(x,y),{a_p}(x,y),{b_p}(x,y))^T}
\end{array} \right.
\end{equation}
Then the final shadow-free image $\bm{u_f}^{Lab}$ in Lab space is obtained as following,
\begin{equation}\label{labunion}
  \bm{u_f}^{Lab}(x,y) = {({L_c}(x,y),{a_p}(x,y),{b_p}(x,y))^T}
\end{equation}

An overview of our shadow-free image generation pipeline is presented in Algorithm \ref{algorithm1}. Fig. \ref{fig:flow_chart} shows the stages of the algorithm
implemented on a shadowed image. More shadow-free images can be seen in the experiment section and $\mathbf{Appendix}$ B. From the results, we can see that our method for shadow-free image effectively removes shadows from images while retaining the color and
texture information of the original image. It is worth noting that,
unlike shadow removal method that only recovers pixels in the shadow region, our shadow-free image, without the process
of shadow detection, is operating on the entire image. In our algorithm, all the brighter pixels
and darker pixels are pulled to the same illumination intensity and it may reduce the contrast of the image. Therefore, our shadow-free image may sometimes looks
a bit dull compared with the original image.

%-------------------------------------------------------------------------
\section{Experiments}
In our experiment, we evaluate the results of color illumination invariant image and the relevant shadow-free image on more outdoor real images. These
images consist of common scenes or objects under a variety of illumination conditions without shadow or with different types of shadows, such as soft
shadows, cast shadows, and self shadows. We first compare our method with two state-of-the-art full-automatic methods (a statistical learning
method \cite{guo2012} and a physical-based method \cite{yang2012} (Figs. \ref{fig:zoom1} and \ref{fig:results3})) and a shadow removal method \cite{arbel2007texture,arbel2011} with human intervention (Fig. \ref{fig:removal}), respectively. Then the results obtained by our method and Shen\textquoteright s method \cite{shen2011} on images under different illumination conditions show that our color illumination invariant image is much closer to the intrinsic
feature than Shen\textquoteright s from the view of yielding an identical reflectivity result regardless of illumination condition (Fig. \ref{fig:intrinsic}).

%-------------------------------------------------------------------------
\subsection{Comparison of shadow-free images}

\textbf{Comparison with full-automatic methods.} In our experiment, we compare our method with two state-of-the-art methods (\hspace{-0.35pt}\cite{guo2012} and \cite{yang2012}) in terms of shadow-free results and running time. These two
methods are based on shadow feature with statistical learning and physical model respectively. These methods are executed with the codes published by
the authors of \cite{guo2012} and \cite{yang2012}.
\begin{figure*}
  \centering
  \begin{center}
  \includegraphics[width=0.88\textwidth]{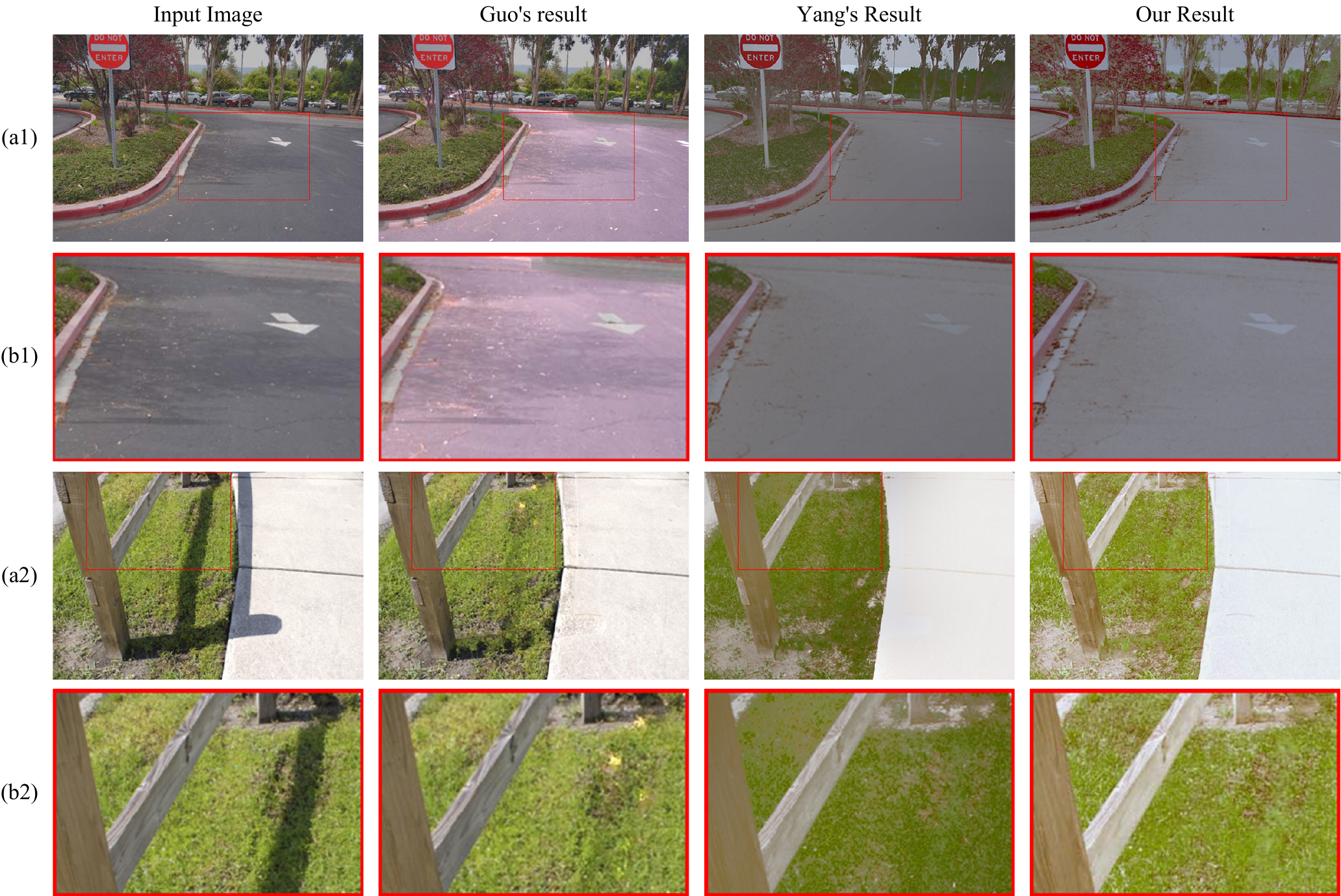}
    \caption{Comparisons with two state-of-the-art methods (Guo et al. \cite{guo2012} and Yang et al. \cite{yang2012}). (b1), (b2) are close-ups of the red rectangles in (a1), (a2).}
  \label{fig:zoom1}
    \end{center}
\end{figure*}
\begin{figure*}
  \centering
  \includegraphics[width=0.88\textwidth]{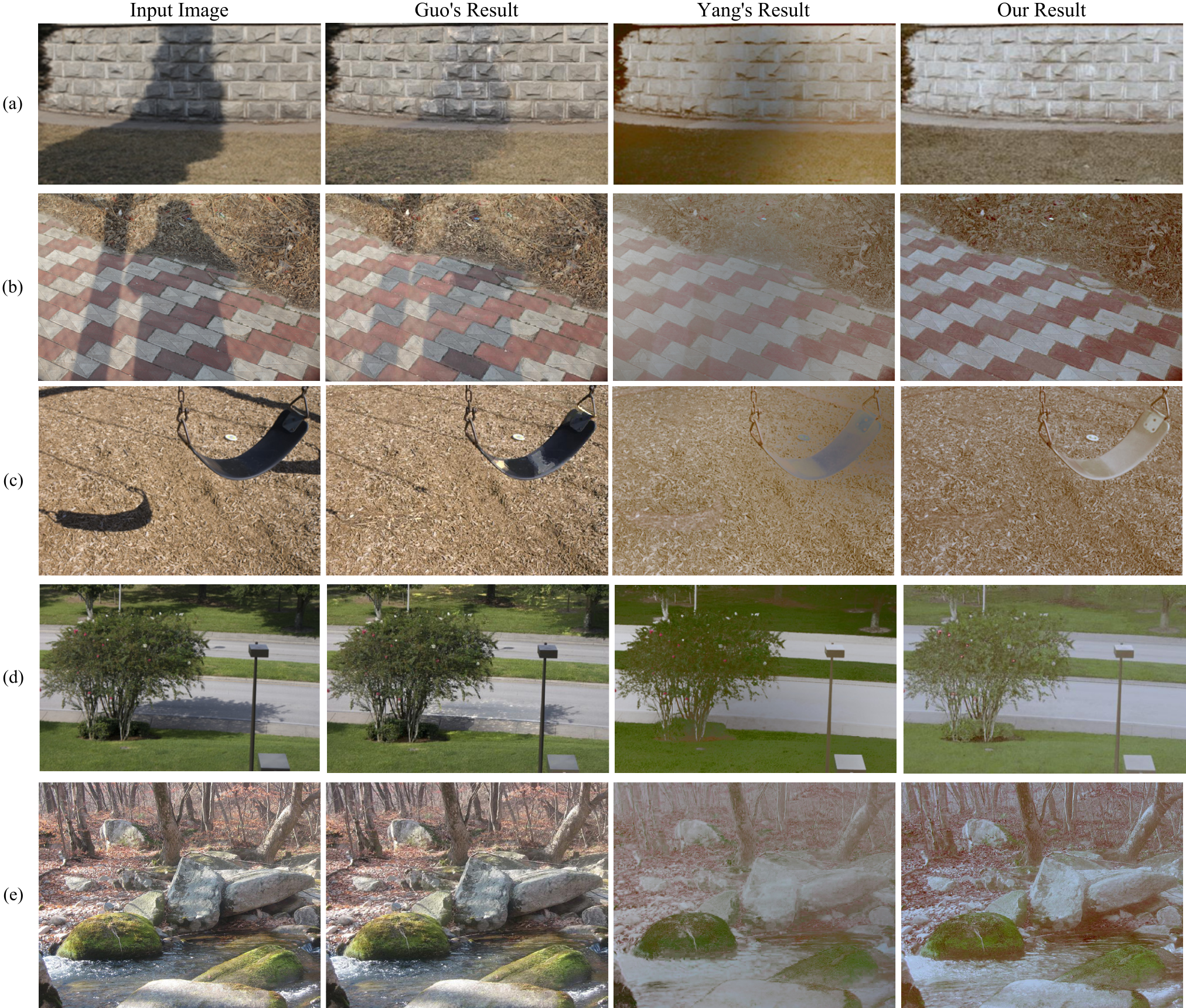}
    \caption{Comparisons with two state-of-the-art methods (Guo et al. \cite{guo2012} and Yang et al. \cite{yang2012}) on more images.}
  \label{fig:results3}
\end{figure*}

In Fig. \ref{fig:zoom1}, we show two shadow-free image results. For the first image (Fig. \ref{fig:zoom1}(a1)), both our method and Yang\textquoteright s
method \cite{yang2012} remove the soft-shadow of leaves successfully. However, Guo\textquoteright s method \cite{guo2012} fails to recover a shadow-free
image due to the wrong detection on shadows. Due to lack of robust shadow features, this shadow feature and learning based method [9] are not always effective,
especially in soft shadow region or some complex scenarios (e.g., Figs. \ref{fig:results3}(d) and \ref{fig:results3}(e)). Close-ups of the first image (Fig. \ref{fig:zoom1}(b1)) show that, compared with our result,
Yang\textquoteright s method \cite{yang2012} fails to retain the white arrow on the ground. This is because that Yang\textquoteright s method is not solid for
neutral regions when its neighbor regions are also neutral \cite{yang2012}.
For the second image (Fig. \ref{fig:zoom1}(a2)), our method gets a shadow-free image with texture preserved, while Guo\textquoteright s method \cite{guo2012}
fails to retain the original texture even when it detects the shadows successfully due to the inappropriate relighting on pixels. Specially, some pixels are relighted
serious wrongly, such that it looks like that some objects that do not exist in the original image appear (e.g., the tiny yellow flowers shown in
Guo\textquoteright s shadow-free image in Fig. \ref{fig:zoom1}(b2)).
Close-ups of the red rectangle (Fig. \ref{fig:zoom1}(b2)) show that some of the
grasses are blurred in Yang\textquoteright s method \cite{yang2012} due to bilateral filtering, while our method, without any filtering operation, retains the texture well.

Some other comparison results are shown in Fig. \ref{fig:results3}. The images in Fig. \ref{fig:results3} show that our method removes shadows effectively from a single
image with texture well-preserved. However, Guo\textquoteright s method \cite{guo2012} fails to recover a shadow-free image either due to the wrong detection on shadow
regions (Figs. \ref{fig:results3}(b), \ref{fig:results3}(d), and \ref{fig:results3}(e)) or the inappropriate relighting on pixels (Figs. \ref{fig:results3}(a) and \ref{fig:results3}(c)).
Due to the complex of image reflectance and the irregular distribution of shadows in Fig. \ref{fig:results3}(e), none of the shadow regions are correctly detected in Guo\textquoteright s
shadow-free image. For images with dark shadows or in which shadows pervade in a scenario (Figs. \ref{fig:results3}(a)--\ref{fig:results3}(e)), Yang\textquoteright s method \cite{guo2012} blurs the shadow-free images. For example,
the clear water in the original image in Fig. \ref{fig:results3}(e) is seriously blurred in Yang\textquoteright s method.
Some small and tiny objects such as leaves are also removed in Yang\textquoteright s results.
\begin{table*}
\centering
\caption{Comparison of running time of the three methods on images shown in Fig. \ref{fig:results3} (measured in seconds)}
\label{table:time}
\begin{tabular}{l||c|c|c|c|c}
\hline
Images & a & b & c & d & e \\
 (Resolution) & (287*163) & (682*453)& (496*334) & (498*332) &(500*475) \\
\hline\hline
Guo\textquoteright s (Matlab)  &	36.36 &	138.63  &  158.2786 & 286.5894 &6.48e+003\\
Ours (Matlab) &	 0.13 &	0.78 &0.384&0.386&0.546\\
\hdashline
Yang\textquoteright s (C++)  &	0.171 &	0.94  & 0.602 &0.614&0.75\\
Ours (C++) &	 0.041 &	0.187 & 0.119 & 0.092 &0.157\\
\hline
\end{tabular}
\end{table*}

Figs. \ref{fig:zoom1} and \ref{fig:results3} verify that our method removes shadows more effectively from a single image, and performs steadily on images with complex outdoor
scenes. Our shadow-free image also better preserves the original texture compared to the other two methods.
In running time, our method also outperforms Guo\textquoteright s \cite{guo2012} and Yang\textquoteright s \cite{yang2012}. In Tab.\ref{table:time},
we compare the running time of these three methods on the images shown in Fig. \ref{fig:results3}. All experiments were conducted
on a computer with Intel (R) Core (TM) i3-2120 3.3GHz CPU and 2G RAM memory. The source codes that Guo and Yang provided are implemented in Matlab and C++ respectively. For better comparison, we implement our method both in Matlab and C++. Without any statistical learning or complex filtering operation, our method is
faster than Guo\textquoteright s \cite{guo2012} and Yang\textquoteright s \cite{yang2012}.
Note that, when the reflectance of image becomes complex and
shadows on the image get intricate, it will be extremely hard for Guo\textquoteright s method to detect shadow regions and then remove them,
and the running time rose up to 6.4827e+003 second, almost 2 hours in Fig. \ref{fig:results3}(e). All these tested examples and quantitative results show that our
method performs better in terms of accuracy, consistency, and efficiency than the other two methods on shadow images.

\textbf{Comparison with shadow removal method with human intervention.} Apart from previous comparisons with full-automatic methods for shadow-free image, in this section, we also compare our method
with a shadow removal method \cite{arbel2007texture,arbel2011}. In Arbel\textquoteright s method, the shadow mask is firstly extracted based on region growing and SVM learning with the user inputs (Figs. \ref{fig:removal}(b) and \ref{fig:removal}(c)).
Then the shadow-free image is constructed by multiplying a constant scalar to
cancel the effect of shadows in umbra regions, while on penumbra regions, the shadows are removed based on the intensity surface construction.
\begin{figure*}[h]
  \centering
  \includegraphics[width=0.89\textwidth, height=0.33\textwidth]{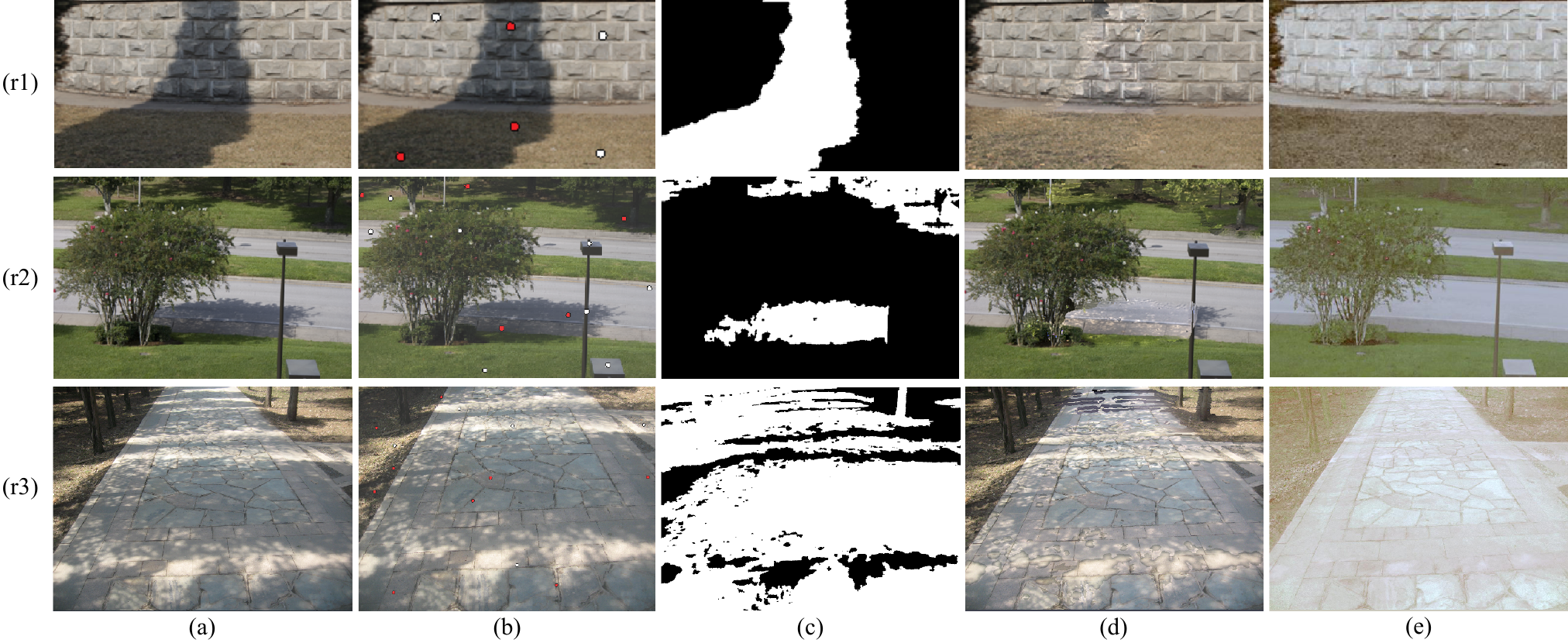}
    \caption{Comparison with Arbel\textquoteright s shadow removal method \cite{arbel2007texture,arbel2011}. (a) Original image. (b) Examples of user-guided shadow
    and non-shadow observations in Arbel\textquoteright s method denoted by red circles and white circles respectively. (c) The extracted shadow masks in Arbel\textquoteright s method. (d) Arbel\textquoteright s shadow-free image. (e) Our shadow-free image.}
  \label{fig:removal}
\end{figure*}

Fig. \ref{fig:removal} gives the comparison of our method with Arbel\textquoteright s
method on shadow-free image.
Fig. \ref{fig:removal}(r1,e) shows that even though Arbel\textquoteright s method effectively removes shadows from
image with dark shadows, the texture in the original shadow area is not consistent
with that in the non-shadow area. Figs. \ref{fig:removal}(r2,d) and \ref{fig:removal}(r3,d) further show that
Arbel\textquoteright s method fails to reconstruct a shadow-free image from images with complex shadows. In these complex scenarios, the constant scalar in the umbra area
in Arbel\textquoteright s method no longer works. These wrong constant scalar relights the shadow regions either excessively or insufficiently.
As a contrast, our method can
not only remove the shadows, but also keep the texture consistency between the
shadow and non-shadow area as shown in Figs. \ref{fig:removal}(r1,d), \ref{fig:removal}(r2,e), and \ref{fig:removal}(r3,e).

%-------------------------------------------------------------------------
\subsection{The ``intrinsic'' of color illumination invariant image}
Both Guo\textquoteright s \cite{guo2012} method, Yang\textquoteright s \cite{yang2012} method, and Arbel\textquoteright s method \cite{arbel2007texture,arbel2011} are designed for removing shadows from an image.
They cannot yield an identical reflectivity result regardless of illumination condition. In this section we will show that our color illumination invariant
image has intrinsic feature. We apply our algorithm to a series of images taken in outdoor scenes at different times on a sunny day to testify the intrinsic feature of our color illumination invariant image.

Shown in Fig. \ref{fig:intrinsic}(r1), these tested images were taken under four different illumination conditions,
one under total daylight (Fig. \ref{fig:intrinsic}(r1,a)), two under partly daylight and partly skylight
(Figs. \ref{fig:intrinsic}(r1,b) and \ref{fig:intrinsic}(r1,c)), and one under total skylight (Fig. \ref{fig:intrinsic}(r1,d)). Taking the image under total daylight as reference image, we compare
the differences between the corresponding processed results by different methods. For an intrinsic image derivation algorithm, the difference of the reference
intrinsic image and other three intrinsic images should be approximately the same or at least much closer to each other than the original image sequences.
In our experiment, we calculated the differences both for our color illumination invariant images and shadow-free images. As a comparison, we also present
the intrinsic images obtained by Shen et al. \cite{shen2011} and the corresponding differences. The
root mean square error (RMSE) is used as the measurement. Since being compressed, the range of our color illumination is much
lower compared with original images and shadow-free images. Therefore, besides RMSE, we also adopt Relative Error
($\frac{{{\rm{RMSE}}}}{{{\rm{ \text{The Range of Mearsued Image}}}}}$) as our measurement.
All these measured images are 24 bit RGB images.
\begin{figure*}[h]
  \centering
\begin{center}
  \includegraphics[width=0.9\textwidth,height=0.58\textwidth]{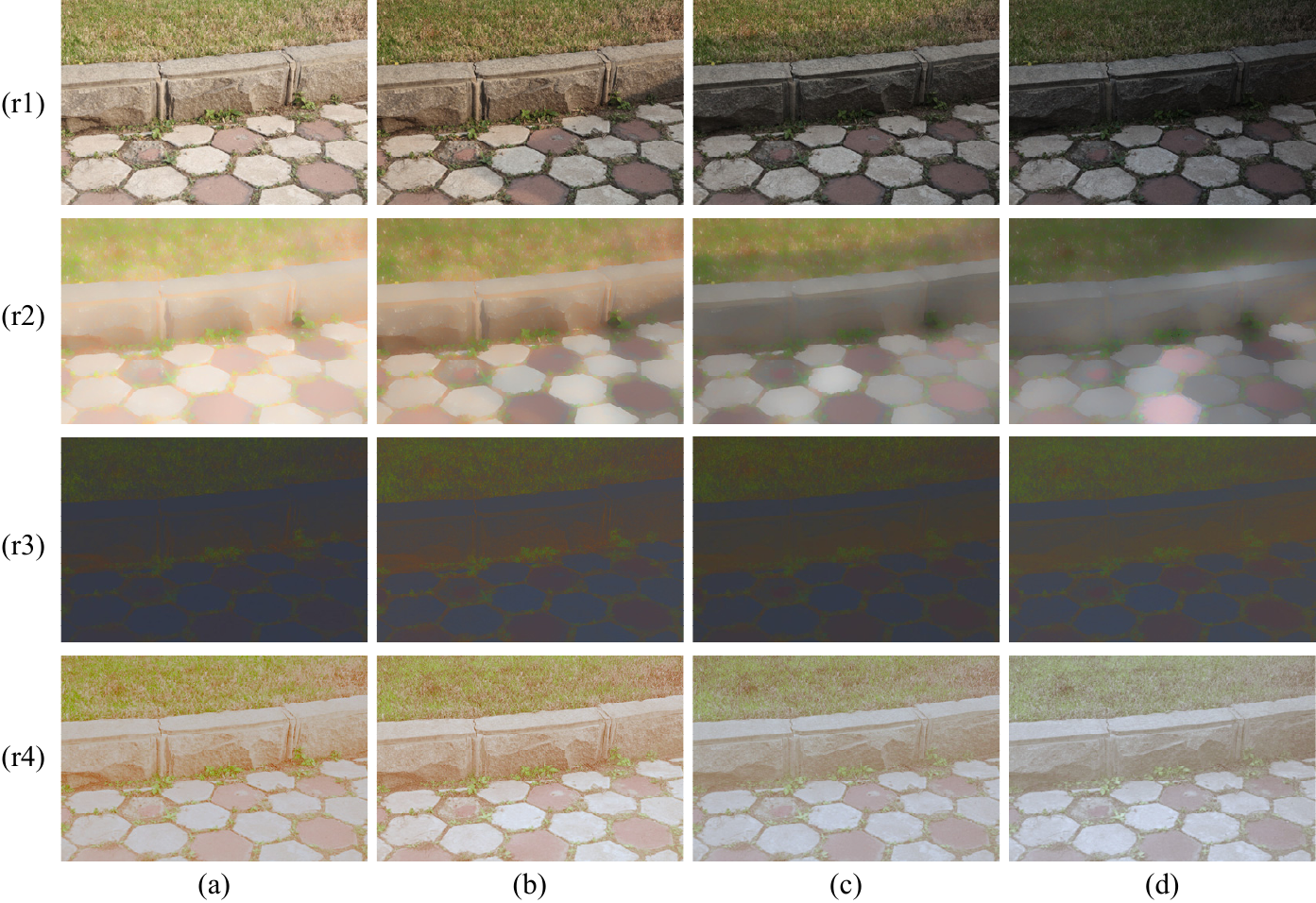}
    \caption{Images under different illumination conditions. For rows: (r1) original image, (r2) intrinsic image by Shen et al. \cite{shen2011},
    (r3) our color illumination invariant image, (r4) our shadow-free image. For columns: (a) (reference image) in daylight,
    (b) and (c) partly in daylight and partly in skylight, (d) in skylight.}
  \label{fig:intrinsic}
  \end{center}
\end{figure*}
\begin{table*}[h]
\renewcommand{\arraystretch}{1.3}
\caption{MSE and relative error between images}
\label{table:error}
\centering
\begin{tabular}{l||c|c|c||c|c|c}
\hline
 & \multicolumn{3}{c||}{RMSE} & \multicolumn{3}{c}{Relative Error (\%)}\\
 \hline
  Images & a and b & a and c &  a and d & a and b & a and c &  a and d \\
\hline\hline
Original img. & 45.96 &80.03 &101.85 & 18.02 &31.39 &39.94 \\
Intrinsic img. by \cite{shen2011}  & 45.17 &	72.67 & 86.36& 17.71 &	28.50 & 33.87 \\
Our invariant img.  & 4.98 &	4.49 &	4.67& 5.79 &	5.22 &	5.70 \\
Our shadow-free img. & 15.12& 17.87 & 24.88 & 6.00& 7.45 & 9.99 \\
\hline
\end{tabular}
\end{table*}

As shown in Fig. \ref{fig:intrinsic}(r3), the four color illumination invariant images generated by our method are basically the same.
It verifies that our orthogonal decomposition on an image yields an identical reflectivity result regardless of illumination condition.
The quantitative measurement is given in Tab.\ref{table:error}. Taken the images under total daylight and total skylight as an example, the RMSE difference
of the original images is decreased from 101.85 to 4.67 in our color illumination invariant images. Also, when considering the Relative Error, this value is decreased
from 39.94\% to 5.7\%. Even after a color restoration operation on illumination
invariant images, the RMSE difference (24.88) between our shadow-free images is still much smaller than original images. However, the intrinsic images obtained
by Shen et al. \cite{shen2011} still have a big difference, 86.36. The quantitative measurements demonstrate that our color illumination invariant images are
considerably closer to generate an intrinsic representation than Shen et al. \cite{shen2011}. The corresponding shadow-free images also maintain this
``intrinsic'' to some extent.

%-------------------------------------------------------------------------
\section{Discussion and conclusion}
In this paper, we propose a novel, effective and fast method to obtain a shadow-free image from a single outdoor image.
Different from state-of-the-art methods that either need shadow detection or statistical learning, our shadow-free image is obtained by a pixel-wise
orthogonal decomposition and color restoration. Also, our shadow-free image better preserves the texture information compared to the state-of-the-art method.
Intrinsic to the illumination condition, our color illumination invariant image is useful for relighting by adjusting $\alpha$ information.
Besides, only requiring pixel-wise calculation, our method for color illumination invariant image and shadow-free image can be further applied to the real-time
applications to resist illumination variation.
\begin{figure*}[h]
  \centering
  \includegraphics[width=0.9\textwidth]{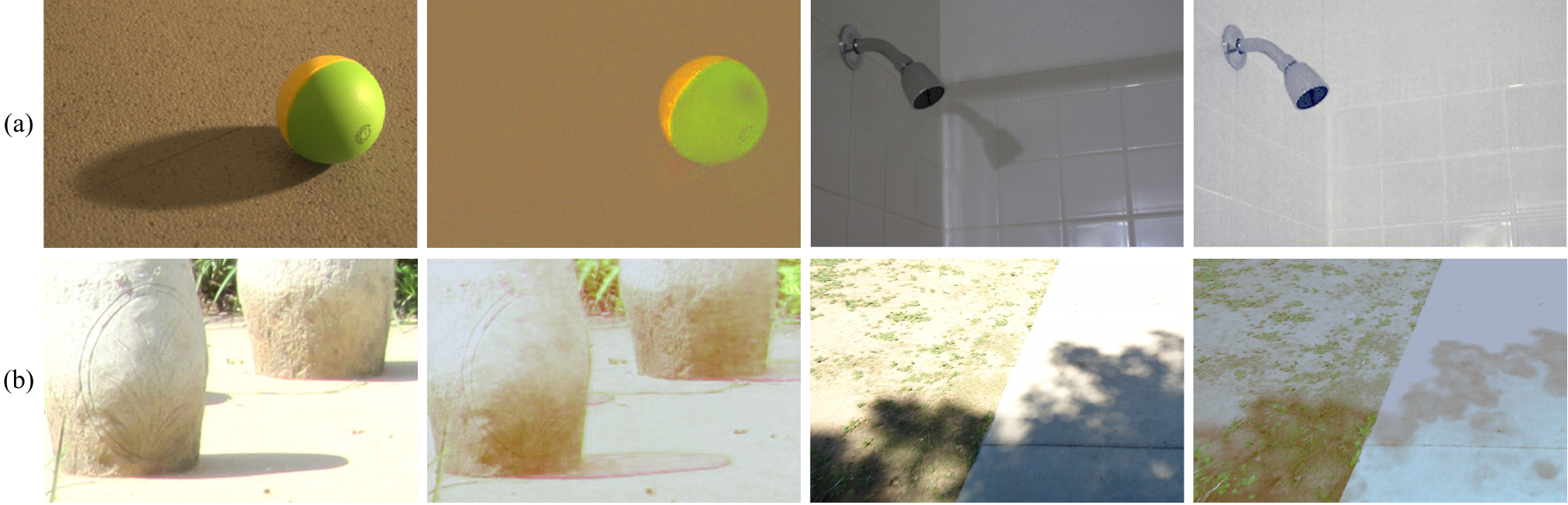}
    \caption{(a) Two original indoor images and the corresponding shadow-free images.
    (b) Two failure cases: original images with over-exposed regions and the corresponding shadow-free images. }
  \label{fig:improvement}
\end{figure*}

Despite that our pixel-wise orthogonal decomposition is designed for outdoor shadow images,
it is also valid for indoor situation by the correction of physical parameters ${\beta_1},{\beta_2},{\beta_3}$. Fig. \ref{fig:improvement}(a) gives
two results for our shadow-free images on the indoor shadow images. It shows that our pixel-wise orthogonal decomposition
method can be also applied in indoor situation.
According to our experiments, we also found two main limitations of our method: 1) Shadow-free image has somewhat color distortion, especially near neutral surface; 2) It does not work well on images with overexposure regions. The main reason for the failure on overexposure regions is: the pixel values on overexposure region are not physically correct values, they do not satisfy the linear model and we
cannot get their correct orthogonal decomposition. Fig. \ref{fig:improvement}(b) shows two failure examples for
our shadow-free images.
The modification and improvement of the model and method for an over-exposed regions, indoor image and more accurate color restoration are in process.

See $\mathbf{Appendix}$ for supporting content.

%\numberwithin{equation}{section}
%\numberwithin{figure}{section}
\appendix
\renewcommand\thesection{\appendixname~\Alph{section}}

\section{The uniqueness of particular solution $\bm{u_p}$}
\newtheorem*{theorem}{The Theorem}
\begin{theorem}
Let $A$ denote a 3 by 3 matrix, $rank(A) = 2 $ and $\bm{b}$  be a non-zero three dimensional vector. For linear equations
\begin{equation}\label{proof1}
  A\bm{u} = \bm{b}
\end{equation}
there exists a unique particular solution $\bm{u_p}$ of Eq. (\ref{proof1}) that is perpendicular to its normalized free solution.
\end{theorem}
\begin{proof}

1) Proof of the existence

Due to $rank(A) = 2 $ and $\bm{b}\neq\bm{0}$, we know that the dimension of the nullspace of Eq. (\ref{proof1}) is one.
Let $\bm{u_s}$ be a particular solution of Eq. (\ref{proof1}) and $\bm{u_0}$ be the normalized free solution of Eq. \ref{proof1}, i.e., $ A\bm{u_s} = \bm{b}$,  $A\bm{u_0} = \bm{0}$ and
${\rm{ }}{\left\| \bm{u_0} \right\|} = 1$.
We obtain a new solution of Eq. (\ref{proof1}) by
\begin{equation}\label{proof2}
  \bm{u_p} = \bm{u_s} - \left\langle {\bm{u_s},\bm{u_0}} \right\rangle \bm{u_0}
\end{equation}
By inner product operation on both sides of Eq. (\ref{proof2}) with $\bm{u_0}$, we have
\[\begin{array}{ll}
\left\langle {\bm{u_p},\bm{u_0}} \right\rangle &= \left\langle {\bm{u_s},\bm{u_0}} \right\rangle  - \left\langle {\bm{u_s},\bm{u_0}}
\right\rangle \left\langle {\bm{u_0},\bm{u_0}} \right\rangle  \\
&= \left\langle {\bm{u_s},\bm{u_0}} \right\rangle  - \left\langle {\bm{u_s},\bm{u_0}} \right\rangle \\
 &= 0
 \end{array}\]
Then we have $\bm{u_p} \bot \bm{u_0}$. Because that the dimension of the nullspace of Eq. (\ref{proof1}) is one,
we get that $\bm{u_p}$ is perpendicular to any non-zero free solution of Eq. (\ref{proof1}).

2) Proof of the uniqueness

Similar due to  $rank(A) = 2 $ and $\bm{b}\neq\bm{0}$,
the solution space of Eq. (\ref{proof1}) can be decomposed into a particular solution plus one-dimensional nullspace solutions.
Without loss of generality, for $\bm{u_p}$ in Eq. (\ref{proof2}) and the normalized free solution $\bm{u_0}$, any solution $\bm{u}$ of Eq. (\ref{proof1}) can be expressed as following,
\begin{equation}\label{proof3}
  \bm{u} = \bm{u_p} + a\bm{u_0}
\end{equation}
where $a \in R$. Suppose there exists another particular solution $\bm{u_p}^\prime$ of Eq. (\ref{proof1}), such that $\bm{u_p}^\prime \bot \bm{u_0}$ and $\bm{u_p}^\prime\neq\bm{u_p}$.
According to Formula \ref{proof3}, we can have
\begin{equation}\label{proof4}
  \bm{u_p}^\prime = \bm{u_p} + a^\prime\bm{u_0}
\end{equation}
By inner product operation on both sides of Formula \ref{proof4} with $\bm{u_0}$, we have
\[\left\langle {\bm{u_p}^\prime,\bm{u_0}} \right\rangle  = \left\langle {\bm{u_p},\bm{u_0}} \right\rangle  + a^\prime\left\langle {\bm{u_0},\bm{u_0}} \right\rangle \]
Because $\bm{u_p} \bot \bm{u_0}$, $\bm{u_p}^\prime \bot \bm{u_0}$, and ${\rm{ }}{\left\| \bm{u_0} \right\|} = 1$, we obtain $0=0+a^\prime$. Then $a^\prime=0$.

From Eq. (\ref{proof4}) and $a^\prime=0$, we can get that the solution of Eq. (\ref{proof1}), which is perpendicular to any non-zero free solution of Eq. (\ref{proof1}), is unique.

\end{proof}

For the case $\bm{b}=0$, we set  $\bm{u_p}=0$.

\section{Our results on more outdoor images}
In this section we show our results on
more outdoor images. For each case, we provide the original image, color illumination invariant image, $\alpha$ information, and the final shadow-free image.

\begin{figure}[!h]
  \centering
  \begin{center}\vspace{-1mm}
  \includegraphics[width=0.90\textwidth,height=1.08588\textwidth]{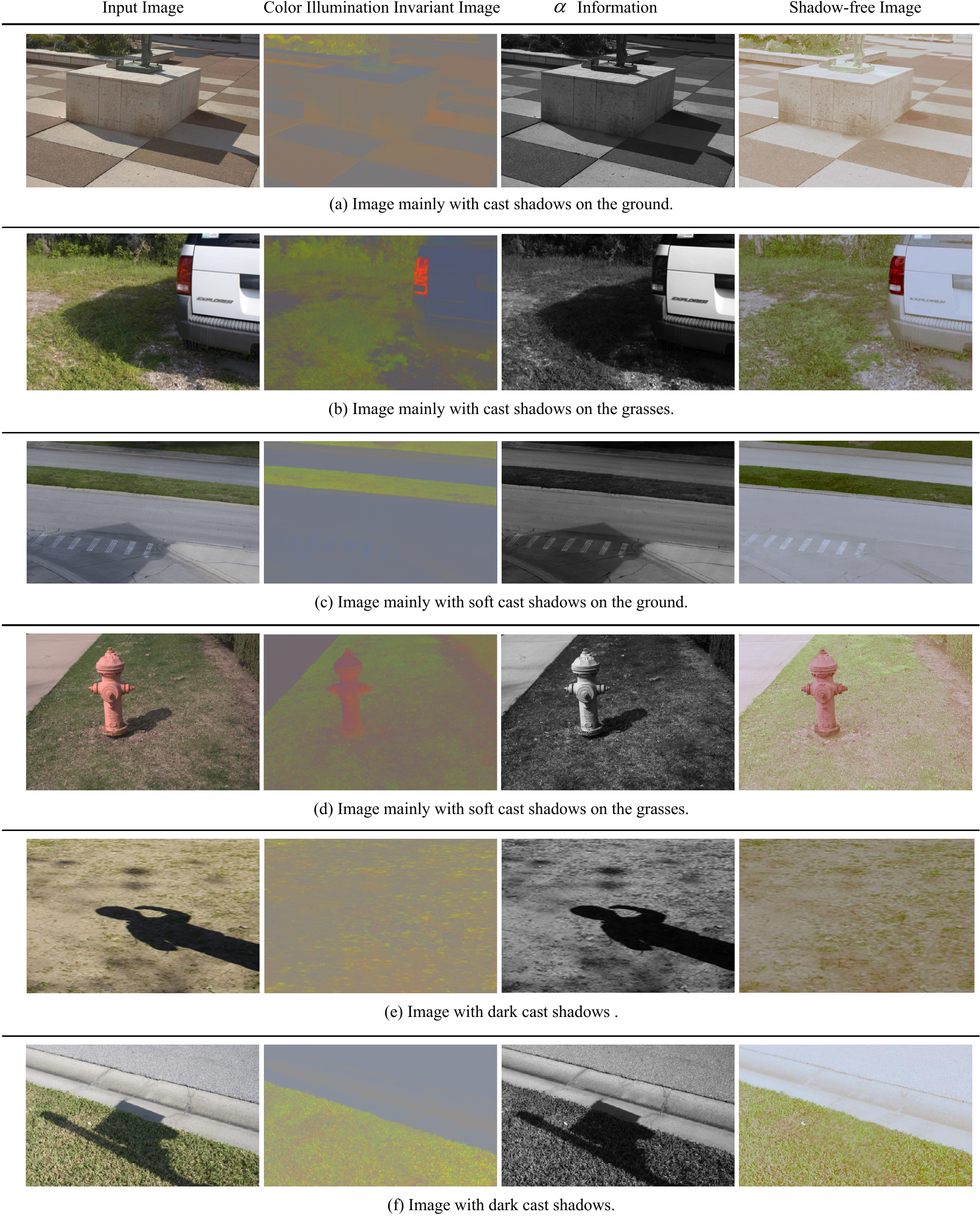}
    \caption{Our results on some cast shadow images. Both images with soft shadows and dark shadows are given.}
  \label{fig:add1}\vspace{-10mm}
    \end{center}\vspace{-7mm}
\end{figure}
 \clearpage
\begin{figure*}[!h]
  \centering
  \begin{center}
  \includegraphics[width=0.92\textwidth]{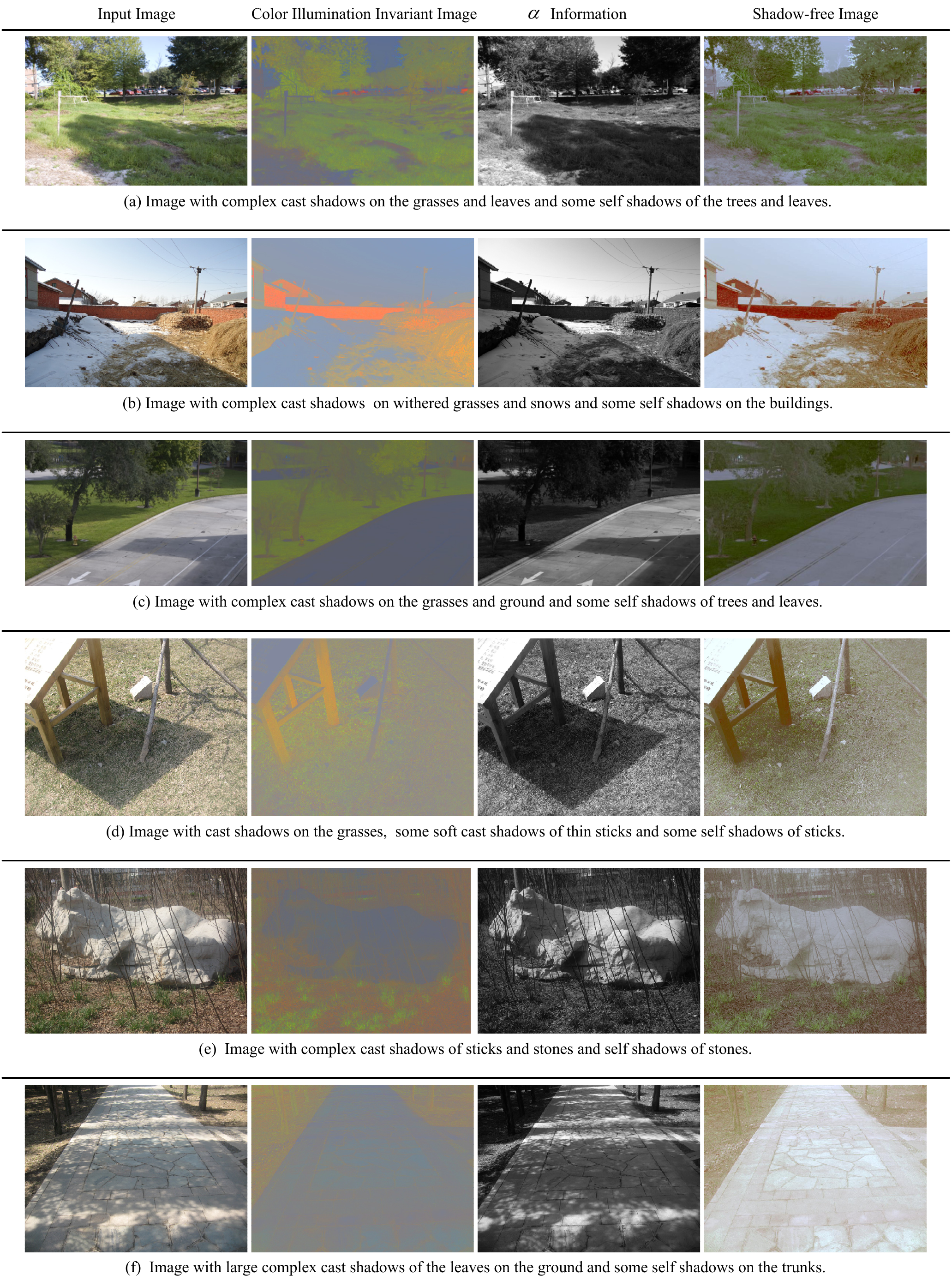}
    \caption{Our results on more complex outdoor scenes, which mainly contain cast shadows
and self shadows.}
  \label{fig:add2}
    \end{center}
\end{figure*}
\section*{Acknowledgments}
This work was supported by the Natural Science Foundation
of China under Grant No.61102116, 61473280, and 61333019. We thank Eli Arbel for his provision of the source code of \cite{arbel2007texture,arbel2011}.

\end{document}